\documentclass{article}

\usepackage{times}
\usepackage{hyperref} 
\usepackage{amsfonts}
\usepackage{amsthm}
\usepackage{graphicx}
\usepackage{setspace}
\usepackage{footnote}

\usepackage{color}
\usepackage{enumerate}
\usepackage[utf8]{inputenc}

\newtheorem{defi}{Definition}
\newtheorem{rmk}{Remark}
\newtheorem{thm}{Theorem}
\newtheorem{prop}{Proposition}
\newtheorem{cor}{Corollary}

\newtheorem{exem}{Example}

\title{A Method for Image Reduction Based on a Generalization of Ordered Weighted Averaging Functions\footnote{Preprint submitted to IEEE Transactions on Fuzzy Systems.}}

\begin{document}

\author{A. Diego S. Farias \footnote{Federal University of Semi-Arid - UFERSA, Pau dos Ferros, RN, Brazil, 59.900-000, antonio.diego@ufersa.edu.br},\ 
 Valdigleis S. Costa \footnote{DIMAp, valdigleis@ppgsc.ufrn.br}, Luiz Ranyer A. Lopes \footnote{DIMAp, ranyer.lopes@gmail.com},\\ 
Benjam\'in Bedregal \footnote{DIMAp, bedregal@dimap.ufrn.br}\ \ and
Regivan H. N. Santiago\footnote{Dimap, regivan@dimap.ufrn.br}
\footnote{DIMAP: Department of Informatics and Applied Mathematics, Federal University of Rio Grande do Norte --- UFRN, Natal, RN, Brazil, 59.072-970}\\}

\maketitle

\begin{abstract}
In this paper we propose a special type of aggregation function which generalizes the notion of Ordered Weighted Averaging Function - $OW\!A$. The resulting functions are called \textbf{Dynamic Ordered Weighted Averaging Functions} --- \textbf{DYOWA}s. This generalization will be developed in such way that the weight vectors are variables depending on the input vector. Particularly, this operators generalize the aggregation functions: {\it Minimum}, {\it Maximum}, {\it Arithmetic Mean}, {\it Median} etc, which are extensively used in image processing. In this field of research two problems are considered: The determination of methods to reduce images and the construction of techniques which provide noise reduction. The operators described here are able to be used in both cases. In terms of image reduction we apply the methodology provided in \cite{Patermain}. We use the noise reduction operators obtained here to treat the images obtained in the first part of the paper, thus obtaining images with better quality.
\end{abstract}


\section{Introduction}

Image processing has great applicability in several areas. In medicine, for example, they can be applied to: Identify tumors \cite{Joseph}; support techniques in advancing dental treatments \cite{Solanki}, etc. Such images are not always obtained with a suitable quality, and to detect the desired information, various methods have been developed in order to eliminate most of the noise contained in these images.

Another problem addressed in image processing is the decrease of resolution, usually aiming the reduction of memory consumption required for its storage \cite{Gonzales}.

There are several techniques for image reduction in the literature, more recently Paternain {\it et. al.} \cite{Patermain} constructed reduction operators using weighted averaging aggregation functions. The proposed method consists of:  (1) To reduce a given image by using a reduction operator; (2) To build a new image from the reduced one, and (3) To analyze the quality of the last image by using the measures {\it PSNR} and {\it MSIM} \cite{Gonzales}.

In this work we introduce a class of aggregation functions called: \textbf{Dynamic Ordered Weighted Averaging Function} - \textbf{(DYOWA)}. They generalize the $OW\!A$ function  introduced by Yager \cite{Yager1} and in particular the operators: {\it Arithmetic Mean}, {\it Median}, {\it Maximum}, {\it Minimum} and $cOW\!A$. We provide a range of their properties such as: idempotence, symmetry and homogeneity as well two methods \footnote{These methods were implemented by using Java 1.8.0$\_$31 software in a 64 bits MS Windows machine.}: \textbf{(1)} for image reduction and \textbf{(2)} for noise treatment.

This paper is structured in the following way: \textsc{Section 2} provides some basics of Aggregation Functions Theory. \textsc{Section 3} introduces Dynamic Ordered Weighted Averaging functions, shows some examples and properties, and introduces a particular $DYOW\!A$ function, called $\mathbf{H}$, which will be fundamental for this work. In \textsc{Section 4} we provide an application of $DYOW\!A$'s to image reduction and in \textsc{Section 5}, we show that $DYOW\!A$ functions are able to treat images with noise, aiming to improve the reduction method used in \textsc{Section 4}. Finally, section \textsc{Section 6} gives the final remarks of this work.

\section{Aggregation Functions}

Aggregation functions are important mathematical tools for applications in several fields: Information fuzzy \cite{Dubois1}; Decision making \cite{Yager2,Zhou,Chen,Bustince,Patermain2}; Image processing \cite{Patermain,Joseph,Beliakov2}  and Engineering \cite{Liang}. In this section we introduce them together with examples and properties. We also present a special family of aggregation functions called \textit{Ordered Weighted Averaging} - \textbf{OWA} and show some of its features.

\subsection{Definition and Examples}\label{subs1}
Aggregation functions associate each entry ${\bf x}$ with $n$ arguments in the closed interval $[0,1]$ an output value also in the interval $[0,1]$; formally we have:

\begin{defi}\hypertarget{def1}
	An $n$-ary aggregation function is a mapping $f:[0,1]^n\rightarrow[0,1]$, which associates each $n$-dimensional vector ${\bf x}$ to a single value $f({\bf x})$ in the interval $[0,1]$ such that:
	\begin{enumerate}
		\item $f(0,...,0)=0$ and $f(1,...,1)=1$;
		\item If ${\bf x}\le {\bf y}$, i.e., $x_i\le y_i$, for all $i=1,2,...,n$, then $f({\bf x})\le f({\bf y})$. 
	\end{enumerate}
\end{defi}

\begin{exem}\hypertarget{exem1}
	\hfill
	\begin{enumerate}[\normalfont (a)]
		\item Arithmetic Mean: $Arith({\bf x}) = \displaystyle\frac{1}{n}(x_{1} + x_{2} ... + x_{n})$
		
		\item Minimum: $Min({\bf x}) = min\{ x_{1}, x_{2},..., x_{n} \}$;
		
		\item Maximum: $Max({\bf x}) = max\{ x_{1}, x_{2},..., x_{n} \}$;
		
		\item Harmonic mean: $f_{n}({\bf x}) = \displaystyle\frac{n}{\frac{1}{x_{1}} + \frac{1}{x_{2}} + \cdots + \frac{1}{x_{n}} }$;
	\end{enumerate}
\end{exem}

From now on we will use the short term ``aggregation"  instead of ``$n$-ary aggregation function".

Aggregations can be divided into four distinct classes: \textit{Averaging, Conjunctive, Disjunctive} and \textit{Mixed}. Since this paper focus on averaging aggregations, we will define only this class. A wider approach in aggregation can be found in \cite{Beliakov,Dubois2,Grabisch,Baczynski}.

\begin{defi}\hypertarget{def2}
	An aggregation is called \textbf{\textit{Averaging}}, if for all ${\bf x}\in [0,1]^n$ we have:
	$$Min({\bf x}) \leq f({\bf x}) \leq Max({\bf x})$$
	
	\begin{exem}\hypertarget{exem2}
		The {\it Arithmetic Mean}, the {\it Maximum} and the {\it Minimum} are averaging aggregation functions.
	\end{exem}	
\end{defi}

\subsection{Special Types Aggregation Functions}

An aggregation function $f$:

\begin{enumerate}
	
	\item[(1)] is {\bf Idempotent} if, and only if, $f(x,..., x) = x$ for all $x \in [0,1]$.
	
	\item[(2)] is {\bf Homogeneous}  of order $k$ if, and only if, for all $\lambda \in [0,1]$ and ${\bf x} \in [0,1]^{n}$, $f(\lambda x_{1},\lambda x_{2},..., \lambda x_{n}) = \lambda^k f(x_{1}, x_{2},..., x_{n})$. When $f$ is homogeneous of order $1$ we simply say that $f$ is homogeneous.
	
	\item[(3)] is {\bf Shift-invariant} if, and only if, $f(x_{1} + r, x_{2} + r,.., x_{n} + r) = f(x_{1}, x_{2},.., x_{n}) + r$, for all $r\in [-1,1]$, ${\bf x} \in [0, 1]^{n}$ such that $(x_1+r,x_2+r,...,x_n+r)\in [0,1]^n$ and $f(x_1,x_2,...,x_n)+r\in [0,1]$.
	
	\item[(4)] is {\bf Monotonic} if, and only if, ${\bf x}\le {\bf y}$ implies $f({\bf x})\le f({\bf y})$.
	
	\item[(5)] is {\bf Strictly Monotone} if, and only if, $f({\bf x})<f({\bf y})$ whenever ${\bf x}<{\bf y}$, i.e. ${\bf x}\le {\bf y}$ and ${\bf x}\ne {\bf y}$.
	
	\item[(6)] has a {\bf Neutral Element} $e \in [0,1]$, if for all $t \in [0,1]$ at any coordinate input vector ${\bf x}$, it has to be:
	$$f(e,..., e, t, e,...,e) = t, \mbox{ and}$$
	
	\item[(7)] $f$ is {\bf Symmetric} if, and only if, its value is not changed under the permutations of the coordinates of ${\bf x}$, i.e, we have:
	$$f(x_{1}, x_{2},..., x_{n}) = f(x_{p_{(1)}}, x_{p_{(2)}}, \cdots, x_{p_{(n)}})$$
	For all $x$ and any permutation $P:\{1,2,...,n\}\rightarrow\{1,2,...,n\}$.
	
	\item[(8)] An {\bf Absorbing Element} \textbf{(\textit{Annihilator})} of an aggregation function $f$, is an element $a \in [0,1]$ such that:
	$$f(x_{1}, x_{2},..., x_{i-1}, a, x_{i+1}, ..., x_{n}) = a$$
	
	\item[(9)] A {\bf Zero Divisor} of an aggregation function is an element $a\in\ ]0,1[$, such that, there is some vector ${\bf x}$ with $x_j>0$, for all $1\le j\le n$, and $f(x_{1},..., x_{j-1}, a, x_{j+1}, .., x_{n}) = 0$. 
	
	\item[(10)] A {\bf One Divisor} of an aggregation function $f$ is an element $a \in\ ]0, 1[$ such that, there is some vector ${\bf x}$ with $x_j<1$, for all $1\le j\le n$, and $f(x_{1},..., x_{j-1}, a, x_{j+1}, .., x_{n}) = 1$. 
	
	\item[(11)] If $N: [0,1] \rightarrow [0,1]$ is a strong negation\footnote{A {\bf strong negation} is an antitonic function $N:[0,1]\rightarrow [0,1]$ such that $N(N(\alpha))=\alpha$ for all $\alpha\in[0,1]$.} and $f: [0,1]^{n} \rightarrow [0,1]$ is an aggregation function, then the {\bf dual aggregation function} of $f$ is:
	$$f^{d}(x_{1}, x_{2},..., x_{n}) = N(f(N(x_{1}),N(x_{2}),..., N(x_{n})))$$
	which is also an aggregation function.
\end{enumerate}

\begin{exem}
	\hfill
	\begin{enumerate}[\normalfont (i)]
		\item The functions: $Arith, Min$ and $Max$ are examples of idempotent, homogeneous, shift-invariant, monotonic and symmetric functions.
		\item $Min$ and $Max$ have $0$ and $1$ elements as annihilator, respectively, but $Arith$ does not have annihiladors.
		\item $Min$, $Max$ and $Arith$ do not have zero divisors and one divisors.
		\item The dual of $Max$ with respect to negation $N(x)=1-x$ is the $Min$ function.
	\end{enumerate}	
\end{exem}

\subsection{Ordered Weighted Averaging Function - OWA}\label{subs3}

In the field of aggregation functions there is a very important subclass in which the elements are parametric; they are called: \textit{\textbf{Ordered Weighted Averaging}} or simply \textbf{OWA} \cite{Yager1}.

An $OW\!A$ is an aggregation function which associates weights to all components $x_{i}$ of an input vector {\bf x}. To achieve that observe the following definition.

\begin{defi}
	Let be an input vector ${\bf x}=(x_1,x_2,\dots,x_n)\in [0,1]^n$ and a vector of weights ${\bf w}=(w_1,\dots, w_n)$, such that $\sum\limits_{i=1}^nw_i=1$. Assuming  the permutation: 
	$$Sort({\bf x})=(x_{p(1)},x_{p(2)},\dots, x_{p(n)})$$
	such that $x_{p(i)}\geq x_{p(i+1)}$, i.e., $x_{p(1)}\geq x_{p(2)}\geq\dots\geq x_{p(n)}$, the Ordered Weighted Averaging ({\it OWA}) Function with respect to ${\bf w}$, is the function $OWA_{\bf w}:[0,1]^n\rightarrow [0,1]$ such that:
	$$OW\!A_{\bf w}({\bf x})=\sum\limits_{i=1}^n w_i\cdot x_{p(i)}$$
\end{defi} 

In what follows we remove ${\bf w}$ from $OW\!A_{\bf w}({\bf x})$. The main properties of such functions are:

\begin{enumerate}
	\item[(a)] For any vector of weights ${\bf w}$, the function $OW\!A({\bf x})$ is idempotent and monotonic. Moreover, $OW\!A({\bf x})$ is strictly increasing if all weights ${\bf w}$ are positive;

	\item[(b)] The dual of a $OW\!A_{\bf w}$ is denoted by $(OW\!A)^d$, with the vector of weights dually ordered, i.e. $(OW\!A_{\bf w})^d=OW\!A_{{\bf w}^d}$, where  ${\bf w}^d = (w_{p(n)}, w_{p(n-1)},..., w_{p(1)})$.
	
	\item[(c)] $OW\!A$  are continuous, symmetric and shift-invariant functions;
	
	\item[(d)] They do not have neutral or absorption elements, except in the special case of functions $OW\!A$ of $Max$ and $Min$.
\end{enumerate}

\subsubsection{Examples of special functions $OW\!A$}

\begin{enumerate}
	\item[1.] If all weight vector components are equal to $\frac{1}{n}$, then $OW\!A({\bf x}) = Arith(({\bf x})$.  
	
	\item[2.] If ${\bf w} = (1, 0, 0,..., 0)$, then $OW\!A({\bf x}) = Max({\bf x})$.
	
	\item[3.] If ${\bf w} = (0, 0, 0, ..., 1)$, then $OW\!A({\bf x}) = Min({\bf x})$.
	
	\item[4.] if $w_{i} = 0$, for all $i$, with the exception of a \textit{k}-th member, i.e, $w_{k} = 1$, then this $OW\!A$ is called {\bf static} and $OW\!A_{\bf w}(x) = x_{k}$
	
	\item[5.] Given a vector ${\bf x}$ and its ordered permutation $Sort({\bf x}) = (x_{(1)},\ldots,x_{(n)})$, the {\it Median} function
	$$Med({\bf x})= \left\{\begin{array}{ll}
	\frac{1}{2}(x_{(k)}+x_{(k+1)}), & \mbox{ if } n=2k\\
	x_{(k+1)}, & \mbox{ if } n=2k+1
	\end{array}\right.$$
	is an $OW\!A$ function in which the vector of weights is defined by:
	 \begin{itemize}
	 	\item If $n$ is odd, then $w_i=0$ for all $i\ne\lceil\frac{n}{2}\rceil$ and $w_{\lceil n/2\rceil}=1$.
	 	\item If $n$ is even, then $w_i=0$ for all $i\ne\lceil\frac{n+1}{2}\rceil$ and $i\ne \lfloor\frac{n+1}{2}\rfloor$, and $w_{\lceil n/2\rceil}=w_{\lfloor n/2 \rfloor}=\frac{1}{2}$.
	 \end{itemize}
\end{enumerate}

\begin{exem}\hypertarget{exem4}
	The $n$-dimensional $cOW\!A$ function \cite{Yager3} is the $OW\!A$ operator, with weighted vector defined by:
	\begin{itemize}
		\item If $n$ is even, then $w_j=\frac{2(2j-1)}{n^2}$, for $1\le j\le \frac{n}{2}$, and $w_{n/2+i}=w_{n/2-i+1}$, for $1\le i\le \frac{n}{2}$. 
		\item If $n$ is odd, then $w_j=\frac{2(2j-1)}{n^2}$, for $1\le j\le \frac{n-1}{2}$, $w_{n/2+i}=w_{n/2-i+1}$, for $1\le i\le \frac{n}{2}$, and $w_{(n+1)/2}=1-2\sum\limits_{j=1}^{(n-1)/2} w_i$.
	\end{itemize}
\end{exem}

The $OW\!A$ functions are defined in terms of a predetermined vector of weights. In the next section we propose the generalization of the concept of $OW\!A$ in order to relax the vector of weights. To achieve that we replace the vector of weights by a family of functions. The resulting functions are called \textbf{Dynamic Ordered Weighted Avegaring Functions} or in short: \textbf{DYOWA}s.

\section{Dynamic Ordered Weighted Avegaring Functions - $DYOW\!A$}

Before defining the notion of $DYOW\!A$ functions, we need to establish the notion of {\it weight-function}.

\begin{defi}\hypertarget{def4}
	A finite family of functions  $\Gamma=\{f_i:[0,1]^n\rightarrow [0,1]\ |\ 1\le i\le n\}$ such that: $$\sum_{i=1}^n f_i({\bf x})=1.$$
	is called family of {\bf weight-function} (FWF).
	
	A {\bf Dynamic Ordered Weighted Averaging} Function or simply {\bf DYOWA} associated to a FWF $\Gamma$ is a function of the form:
	$$DYOW\!A_{\Gamma}({\bf x})=\sum_{i=1}^{n} f_i({\bf x})\cdot x_i$$
\end{defi}

Below we show some examples of $DYOW\!A$ operators with their respective weight-functions.

\begin{exem}\hypertarget{exem5}
	Let be $\Gamma=\{ f_i({\bf x})=\frac{1}{n}\ | \ 1\leq i\leq n\}$. The $DYOW\!A$ operator associated to $\Gamma$, $DYOW\!A_{\Gamma}({\bf x})$, is $Arith({\bf x})$.
\end{exem}

\begin{exem}\hypertarget{exem6}
	The function {\it Minimum} can be obtained from $\Gamma=\{ f_i\ |\ 1\le i\le n\}$, where $f_1({\bf x})=f_2({\bf x})=\cdots =f_{n-1}({\bf x})=0$ and $f_n({\bf x})=1$, for all ${\bf x}\in [0,1]^n$.
\end{exem}

\begin{exem}\hypertarget{exem7}
	Similarly, the function {\it Maximum} is also of type $DYOW\!A$ with $\Gamma$ dually defined.
\end{exem}

\begin{exem}\hypertarget{exem8}
	For any vector of weights ${\bf w}=(w_1,w_2,...,w_n)$, A function $OW\!A_{\bf w}({\bf x})$ is a $DYOW\!A$ in which the weight-functions are given by: $f_i({\bf x})= w_{p(i)}$, where $p:\{1,2,\cdots,n\}\longrightarrow \{1,2,\cdots,n\}$ is the permutation, such that $p(i)=j$ with $x_i=x_{(j)}$. For example: If ${\bf w}=(0.3,0.4,0.3)$, then for ${\bf x}=(0.1,1.0,0.9)$ we have $x_1=x_{(3)},\ x_2=x_{(1)}$ and $x_3=x_{(2)}$. Thus, $f_1({\bf x})=0.3,\ f_2({\bf x})=0.3$, $f_3({\bf x})=0.4$, and $DYOW\!A({\bf x})=0.3\cdot 0.1+ 0.3\cdot1.0 + 0.4\cdot 0.9=0.69$
\end{exem}

\begin{rmk}
	Example 8 shows that the functions $OW\!A$, introduced by Yager, are special cases of $DYOW\!A$ functions. There are, however, some $DYOW\!A$ functions which are not $OW\!A$.
\end{rmk}

\begin{exem}\hypertarget{exem9}
	Let $\Gamma=\{\sin(x)\cdot y, 1-\sin(x)\cdot y\}$. The respective $DYOW\!A$ function is $DYOW\!A(x,y)=(\sin(x)\cdot y)\cdot x+(1-\sin(x)\cdot y)\cdot y$, which is not an $OW\!A$ function.
\end{exem}

\subsection{Properties of $DYOW\!A$ Functions}

The next theorem characterizes the $DYOW\!A$ functions which are also aggregations.

\begin{thm}\hypertarget{teo1}
	Let $\Gamma=\{f_1,\cdots,f_n \}$ be a FWF. A $DYOW\!A_{\Gamma}$ is an aggregation function if, and only if, it is monotonic.
\end{thm}
\begin{proof}
	Obviously, if $DYOW\!A_{\Gamma}$ is an aggregation, then it is monotonic function. Conversely, if $DYOW\!A_{\Gamma}$ is monotonic, then for it to become an aggregation, enough to show that
	$$DYOW\!A_{\Gamma}(0,...,0)=0 \mbox{ e } DYOW\!A_{\Gamma}(1,...,1)=1,$$
	this follows from the definition of $DYOW\!A$.
\end{proof}

\begin{cor}\hypertarget{cor1}
	A $DYOW\!A$ is an aggregation function if, and only if, it is an a aggregation of type averaging.
\end{cor}
\begin{proof}
	For all ${\bf x}=(x_1,...,x_n)$ have to
	$$Min({\bf x})\le x_i\le Max({\bf x}),\ \forall i=1,2,...,n.$$
	So,
	$$\sum\limits_{i=1}^nf_i({\bf x})\cdot Min({\bf x})\le \sum\limits_{i=1}^nf_i({\bf x})\cdot x_i\le \sum\limits_{i=1}^nf_i({\bf x})\cdot Max({\bf x}),$$
	but as $\sum\limits_{i=1}^nf_i({\bf x})=1$, it follows that
	$$Min({\bf x})\le \sum\limits_{i=1}^nf_i({\bf x})\cdot x_i \le Max({\bf x})$$
\end{proof}

\begin{cor}\hypertarget{cor2}
	All functions of the type $DYOW\!A$ presented in examples \hyperlink{exem4}{4}, \hyperlink{exem5}{5}, \hyperlink{exem6}{6}, \hyperlink{exem7}{7} and \hyperlink{exem8}{8} are averaging aggregation functions.
\end{cor}
\begin{proof}
	Just see that those functions are monotonic.
\end{proof}

\begin{prop}\hypertarget{prop1}
	For every $\Gamma$, $DYOW\!A_{\Gamma}$ is idempotent.
\end{prop}
\begin{proof}
	If ${\bf x}=(x,...,x)$ with $t\in [0,1]$, then:
	$$DYOW\!A_{\Gamma}({\bf x})=\sum\limits_{i=1}^nf_i({\bf x})\cdot x=x\cdot\sum\limits_{i=1}^nf_i({\bf x})=x$$
\end{proof}

This property is important because it tells us that every $DYOW\!A$ is idempotent, regardless it is an aggregation or not.

\begin{prop}\hypertarget{prop2}
	If $\Gamma$ is invariant under translations, i.e, $f_i(x_1+\lambda,x_2+\lambda,...,x_n+\lambda)=f_i(x_1,x_2,...,x_n)$ for any ${\bf x}\in [0,1]^n$,  for $i\in\{1,2,\cdots,n\}$ and $\lambda\in[-1,1]$, then $DYOW\!A_{\Gamma}$ is shift-invariant.
\end{prop}
\begin{proof}
	Let ${\bf x}=(x_1,...,x_n)\in[0,1]^n$ and $\lambda\in[-1,1]$ such that $(x_1+\lambda,x_2+\lambda,...,x_n+\lambda)\in[0,1]^n$. then,
	\begin{eqnarray*}
		& DYOW\!A_{\Gamma} & \!\!\!\!\!(x_1+\lambda,...,x_n+\lambda) =\\
		& = & \sum\limits_{i=1}^n f_i(x_1+\lambda,...,x_n+\lambda)\cdot(x_i+\lambda)\\
		& = & \sum\limits_{i=1}^nf_i(x_1+\lambda,...,x_n+\lambda)\cdot x_i\\
		& & +\ \sum\limits_{i=1}^nf_i(x_1+\lambda,...,x_n+\lambda)\cdot\lambda\\
		& = & \sum\limits_{i=1}^nf_i(x_1,...,x_n)\cdot x_i+\lambda\\
		& = & DYOW\!A_{\Gamma}(x_1,...,x_n)+\lambda
	\end{eqnarray*}
\end{proof}

\begin{prop}\hypertarget{prop3}
	If $\Gamma$ is homogeneous of order $k$ (i.e., if $f_i$ is homogeneus of order $k$, for each $f_i\in \Gamma$), then $DYOW\!A_{\Gamma}({\bf x})$ is homogeneous of order $k+1$.
\end{prop}
\begin{proof}
	Of course that, if $\lambda=0$, then $DYOW\!A_{\Gamma}(\lambda x_1,...,\lambda x_n)=\lambda f(x_1,...,x_n$). Now, to $\lambda \ne 0$ we have:
	\begin{eqnarray*}
		DYOW\!A_{\Gamma}(\lambda x_1,...,\lambda x_n) & = & \sum\limits_{i=1}^nf_i(\lambda x_1,...,\lambda x_n)\cdot\lambda x_i\\
		& = & \lambda\cdot \sum\limits_{i=1}^n\lambda^kf_i(x_1,...,x_n)x_i\\
		& = & \lambda^{k+1} \cdot DYOW\!A_{\Gamma}(x_1,...,x_n)
	\end{eqnarray*}
\end{proof}

\begin{exem}\hypertarget{exem10}
	Let $\Gamma$ be defined by
	$$f_i(x_1,...,x_n)=\left\{\begin{array}{ll}
	\frac{1}{n}, & \mbox{ if } x_1=...=x_n=0\\
	\frac{x_i}{\sum\limits_{j=1}^n x_j}, & \mbox{otherwise}
	\end{array}\right.$$
	Then, $$DYOW\!A_{\Gamma}({\bf x})=\left\{\begin{array}{ll}
	0, & \mbox{ if } x_1,...,x_n=0\\
	\frac{\sum\limits_{i=1}^n x_i^2}{\sum\limits_{i=1}^n x_i}, & \mbox{otherwise} \end{array}\right.$$
	This $DYOW\!A_{\Gamma}$ is idempotent, homogeneous and shift-invariant. However, $DYOW\!A_{\Gamma}$ is not monotonic, since $DYOW\!A_{\Gamma}(0.5,0.2,0.1)=0.375$ and $DYOW\!A_{\Gamma}(0.5,0.22,0.2)=0.368$.
\end{exem}

The next definition provides a special FWF, which will be used to build a $DYOW\!A $ whose properties are very important for this paper.

\begin{defi}\hypertarget{def5}
	Consider the family $\Gamma$ of functions $$f_i({\bf x})=\left\{\begin{array}{ll}
	\frac{1}{n}, &\!\!\!\!\! \mbox{if } {\bf x}=(x,...,x)\\ \frac{1}{n-1}\left(1-\frac{|x_i-Med({\bf x})|}{\sum\limits_{j=1}^n|x_j-Med({\bf x})|}\right), & \!\!\!\!\mbox{otherwise}\end{array}\right.$$
		Then, $\Gamma$ is a FWF, i.e. $\sum\limits_{i=1}^n f_i({\bf x})=1$, for all ${\bf x}\in [0,1]^n$. Let $\mathbf{H}$ be the associated $DYOW\!A$. The computation of $\mathbf{H}$ can be performed using the following expressions:
		{ \small \begin{eqnarray*}
				\mathbf{H} ({\bf x}) & = & \left\{ \begin{array}{ll}
					x, &\!\!\! \mbox{ if } {\bf x}=(x,...,x)\\
					\frac{1}{n-1}\sum\limits_{i=1}^n\left(x_i-\frac{x_i|x_i-Med({\bf x})|}{\sum\limits_{j=1}^n|x_j-Med({\bf x})|} \right), & \!\!\! \mbox{otherwise}\end{array}\right.
			\end{eqnarray*}}
\end{defi}
	
\begin{exem}\hypertarget{exem11}
	Let be $n=3$. So, for ${\bf x}=(0.1,0.3,0)$ we have
	$$f_1({\bf x})=0.5,\ f_2({\bf x})=0.167,\ f_3({\bf x})=0.333 \mbox{ and } \mathbf{H}({\bf x})=0.08.$$
\end{exem}

\begin{prop}\hypertarget{prop4}
	The weight-functions defined in \hyperlink{def5}{Definition 5} are such that: $f_i(x_1+\lambda,...,x_n+\lambda)=f_i(x_1,x_2,...,x_n)$ and $f_i(\lambda x_1,...,\lambda x_n)=f_i(x_1,...,x_n)$, for any $1\le i\le n$.
\end{prop}
\begin{proof}
	Writing ${\bf x'}=(x_1+\lambda,...,x_n+\lambda)$, then $f(x_1+\lambda,...,x_n+\lambda) =(f_1({\bf x'}),...,f_n({\bf x'}))$. Clearly, $Med({\bf x'})=Med({\bf x})+\lambda$. Thus, for ${\bf x}\ne(x,...,x)$ we have:
	\begin{eqnarray*}
		f_i({\bf x'}) & = & \textstyle\frac{1}{n-1}\left(1-\frac{|x_i+\lambda-Med({\bf x'})|}{\sum\limits_{j=1}^n|x_j+\lambda-Med({\bf x'})|}\right)\\
		& = & \textstyle\frac{1}{n-1}\left(1-\frac{|x_i+\lambda-(Med({\bf x})+\lambda)|}{\sum\limits_{j=1}^n|x_j+\lambda-(Med({\bf x})+\lambda)|}\right)\\
		& = & \textstyle\frac{1}{n-1}\left(1-\frac{|x_i-Med({\bf x})|}{\sum\limits_{j=1}^n|x_j-Med({\bf x})|}\right)\\
		& = & f_i({\bf x}).
	\end{eqnarray*}
	Therefore, $f({\bf x'})=(f_1({\bf x'}),...,f_n({\bf x'}))=(f_1({\bf x}),...,f_n({\bf x}))$. The case in which ${\bf x} =(x,...,x)$ is immediate.
	
	To check the second property, make ${\bf x''}=(\lambda x_1,...,\lambda x_n)$, note that $Med({\bf x''})=\lambda med ({\bf x})$ and for ${\bf x} \ne (x,...,x)$	
	\begin{eqnarray*}
		f_i({\bf x''}) & = & \textstyle\frac{1}{n-1}\left(1-\frac{|\lambda x_i-Med({\bf \lambda x})|}{\sum\limits_{j=1}^n|\lambda x_j-Med({\bf \lambda x})|}\right)\\
		& = & \textstyle\frac{1}{n-1}\left(1-\frac{|\lambda x_i-\lambda Med({\bf x})|}{\sum\limits_{j=1}^n|\lambda x_j-\lambda Med({\bf x})|}\right)\\
		& = & \textstyle\frac{1}{n-1}\left(1-\frac{|\lambda|\cdot|x_i-Med({\bf x})|}{|\lambda|\cdot\sum\limits_{j=1}^n|x_j-Med({\bf x})|}\right)\\
		& = & \textstyle\frac{1}{n-1}\left(1-\frac{|x_i-Med({\bf x})|}{\sum\limits_{j=1}^n|x_j-Med({\bf x})|}\right)\\
		& = & f_i({\bf x})
	\end{eqnarray*}
	Therefore, $f({\bf x''})=(f_1({\bf x''}),...,f_n({\bf x''}))=(f_1({\bf x}),...,f_n({\bf x}))=f({\bf x})$. The case in which ${\bf x} =(x,...,x)$ is also immediately
\end{proof}

\begin{cor}\hypertarget{cor3}
	$\mathbf{H}$ is shift-invariant and homogeneous.  
\end{cor}
\begin{proof}
 Straightforward for propositions \hyperlink{prop2}{2} and \hyperlink{prop3}{3}.
\end{proof}

The function $\mathbf {H}$ is of great importance to this work, since this function, as well as some $DYOW\!A$'s already mentioned will provide us tools able: (1) To reduce the size of images  and (2) To deal with noise reduction.

Now, we present some other properties of function $\mathbf{H}$.

\subsection{Properties of $\mathbf{H}$}

In addition to idempotency, homogeneity and shift-invariance $\mathbf{H}$ has the following proprerties.

\begin{prop}\hypertarget{prop5}
	$\mathbf{H}$ has no neutral element.
\end{prop}
\begin{proof}
	Suppose $\mathbf{H}$ has a neutral element $e$, find the vector of weight for ${\bf x}=(e,...,e,x,e,...,e)$. Note that if $n\geq 3$, then $Med({\bf x})=e$ and therefore,
	\begin{eqnarray*}
		f_i({\bf x}) & = & \textstyle\frac{1}{n-1}\left( 1-\frac{|x_i-Med({\bf x})|}{\sum\limits_{j=1}^n|x_j-Med({\bf x})|}\right)\\
		& = & \textstyle\frac{1}{n-1}\left(1-\frac{|x_i-e|}{\sum\limits_{j=1}^n|x_j-e|}\right)\\
		& = & \textstyle\frac{1}{n-1}\left(1-\frac{|x_i-e|}{|x-e|}\right)
	\end{eqnarray*}
	therefore,
	$$f_i({\bf x})=	\displaystyle\left\{\begin{array}{ll}
	\frac{1}{n-1}, \mbox{ if } x_i=e\\
	0, \mbox{ if } x_i=x
	\end{array}\right., \mbox{ to } n\geq 3$$
	i.e.,
	$$f({\bf x})=\textstyle\left(\frac{1}{n-1},...,\frac{1}{n-1},0,\frac{1}{n-1},...,\frac{1}{n-1} \right)$$
	and
	$$\mathbf{H}({\bf x})=(n-1)\cdot\frac{e}{n-1}=e$$
	But since $e$ is a neutral element of $\mathbf{H}$, $\mathbf{H}({\bf x})=x$. Absurd, since we can always take  $x\ne e$.
	
	For $n=2$, we have $Med({\bf x})=\displaystyle\frac{x+e}{2}$, where ${\bf x}=(x,e)$ or ${\bf x}=(e,x)$. In both cases it is not difficult to show that $f({\bf x})=(0.5,0.5)$ and $\mathbf{H}({\bf x})=\displaystyle\frac{x+e}{2}$. Thus, taking $x\ne e$, again we have $\mathbf{H}(x,e)\ne x$.
\end{proof}

\begin{prop}\hypertarget{prop6}
	$\mathbf{H}$ has no absorbing elements.
\end{prop}
\begin{proof}
	To $n=2$, we have $\mathbf{H}({\bf x})=\displaystyle\frac{x_1+x_2}{2}$, which has no absorbing elements. Now for $n\geq 3$ we have to ${\bf x}=(a,0,...,0)$ with $Med({\bf x})=0$ therefore,
	$$f_1({\bf x})=\frac{1}{n-1}\left(1-\frac{a}{a} \right)=0 \mbox{ and } f_i=\frac{1}{n-1}, \forall i=2,...,n.$$
	therefore,
	$$\mathbf{H}(a,0,...,0)=0\cdot a+\frac{1}{n-1}\cdot 0+...+\frac{1}{n-1}\cdot 0=a\Rightarrow a=0,$$
	but to ${\bf x}=(a,1,...,1)$ we have to $Med({\bf x})=1$. Furthermore,
	$$f_1({\bf x})=\frac{1}{n-1}\left(1-\frac{1-a}1-{a} \right)=0 $$
	and
	$$f_i=\frac{1}{n-1} \mbox{ para } i=2,3,...,n.$$
	therefore,
	$$\mathbf{H}(a,1,...,1)=0\cdot a+\frac{1}{n-1}\cdot 1+...+\frac{1}{n-1}\cdot 1=a\Rightarrow a=1.$$
	With this we prove that $\mathbf{H}$ does note have annihiladors.
\end{proof}

\begin{prop}\hypertarget{prop7}
	$\mathbf{H}$ has no zero divisors.
\end{prop}
\begin{proof}
	Let $a\in\ ]0,1[$ and consider ${\bf x}=(a,x_2,...,x_n)\in\ ]0,1]^n$. In order to have $\mathbf{H}({\bf x})=\sum\limits_{i=1}^n f_i({\bf x})\cdot x_i=0$ we have $f_i({\bf x})\cdot x_i=0$ for all $i=1,2,...,n$. But as $a\ne 0$ and we can always take $x_2,x_3,...,x_n$ also different from zero, then for each $i=1,2,...,n$ there remains only the possibility of terms:
	$$f_i{(\bf x)}=0 \mbox{ para } i=1,2,...,n.$$
	This is absurd, for $f_i({\bf x})\in [0,1]$ e $\sum\limits_{i=1}^nf_i({\bf x})=1$. like this, $\mathbf{H}$ has no zero divisors.
\end{proof}

\begin{prop}\hypertarget{prop8}
	$\mathbf{H}$ does not have one divisors
\end{prop}
\begin{proof}
	Just to see that $a\in\ ]0,1[$, we have to $\mathbf{H}(a,0,...,0)=f_1({\bf x}).a\le a<1$.
\end{proof}

\begin{prop}\hypertarget{prop9}
	$\mathbf{H}$ is symmetric.
\end{prop}

\begin{proof}
	Let $P:\{1,2,...,n\}\rightarrow\{1,2,...,n\}$ be a permutation. So we can easily see that $Med(x_{P(1)},x_{P(2)},...,x_{P(n)})=Med(x_1,x_2,...,x_n)$ for all ${\bf x}=(x_1,x_2,...,x_n)\in[0,1]^n$. We also have to $\sum\limits_{i=1}^n |x_{P(i)}-Med(x_{P(1)},x_{P(2)},...,x_{P(n)})|=\sum\limits_{i=1}^n|x_i-Med({\bf x})|$. Thus, it suffices to consider the case where $(x_{P(1)},x_{P(2)},...,x_{P(n)})\ne (x,x,...,x)$. But $(x_{P(1)},x_{P(2)},...,x_{P(n)})\ne (x,x,...,x)$ we have to:
	\begin{eqnarray*}
		\!\!\! & \mathbf{H} & \!\!\!\!\!(x_{P(1)},x_{P(2)},...,x_{P(n)}) =\\
		\!\!\!\!\!\!\! & = &\!\!\!	\textstyle\frac{1}{n-1}\sum\limits_{i=1}^n\left(x_{P(i)}-\frac{x_{P(i)}|x_{P(i)}-Med(x_{P(1)},...,x_{P(n)})|}{\sum\limits_{j=1}^n|x_{P(i)}-Med(x_{P(1)},...,x_{P(n)})|} \right)\\
		\!\!\!\!\! & = &\!\!\! \textstyle\frac{\sum\limits_{i=1}^n x_{P(i)}}{n-1}-\frac{1}{n-1}\cdot\sum\limits_{i=1}^n \frac{x_{P(i)}|x_{P(i)}-Med(x_1,...,x_n)|}{\sum\limits_{j=1}^n|x_{P(i)}-Med(x_1,...,x_n)|}\\
		& = &\!\!\! \textstyle\frac{\sum\limits_{i=1}^n x_i}{n-1}-\frac{1}{n-1}\cdot\sum\limits_{i=1}^n \frac{x_{P(i)}|x_{P(i)}-Med(x_1,...,x_n)|}{\sum\limits_{j=1}^n|x_i-Med(x_1,...,x_n)|}\\
		& = &\!\!\! \textstyle\frac{\sum\limits_{i=1}^n x_i}{n-1}-\frac{1}{n-1}\cdot\sum\limits_{i=1}^n \frac{x_i|x_i-Med(x_1,...,x_n)|}{\sum\limits_{j=1}^n|x_i-Med(x_1,...,x_n)|}\\
		& = &\!\!\! \mathbf{H}(x_1,...,x_n).
	\end{eqnarray*}
\end{proof}

Therefore, $\mathbf{H}$ satisfies the following properties:
\begin{itemize}
	\item Idempotence
	\item Homogeneity
	\item Shift-invariance
	\item Symmetry.
	\item $\mathbf{H}$ has no neutral element
	\item $\mathbf{H}$ has no absorbing elements
	\item $\mathbf{H}$ has no zero divisors
	\item $\mathbf{H}$ does not have one divisors
\end{itemize}

\begin{rmk}
	Unfortunately we do not prove here the monotonicity of $\mathbf{H}$, due to its complexity, but we suspect that it is true. This demonstration will be relegated to a future work.	
\end{rmk}

The next two sections show the suitability of $DYOW\!A$. They will provide applications for image and noise reduction.

\section{$DYOW\!A$'s as images reduction tools}

In this part of our work we use the functions $DYOW\!A$ studied in Examples \hyperlink{exem4}{4}, \hyperlink{exem5}{5}, \hyperlink{exem6}{6}, \hyperlink{exem7}{7} and \hyperlink{exem8}{8}, and definition \hyperlink{def5}{5} to build image reduction operators, the resulting images will be compared with the reduced image obtained from the operator function $\mathbf{H}$.

An image is a matrix $m\times n$, $M=A(i,j)$, where each $A(i,j)$ represents a pixel. In essence, a reduction operator reduces a given image $m\times n$ to another $m'\times n'$, such that $m'<m$ and $n'<n$. For example,
$$\left[\begin{tabular}{cccc}
0.1 & 0.2 & 0 & 0.5\\
0.3 & 0.3 & 0.2 & 0.8\\
1 & 0.5 & 0.6 & 0.4\\
0 & 0.3 & 0.5 & 0.7\\
\end{tabular}\right]\longmapsto \left[\begin{tabular}{cc}
0.1 & 0\\
1 & 0.6\\
\end{tabular}\right]$$

In Grayscale images the value of pixels belong to the set $[0,255]$, which can be normalized by dividing them by $255$, so that we can think of pixels as values in the range $[0,1]$.

\begin{figure}[h]
	\centering
	\includegraphics[width=7cm]{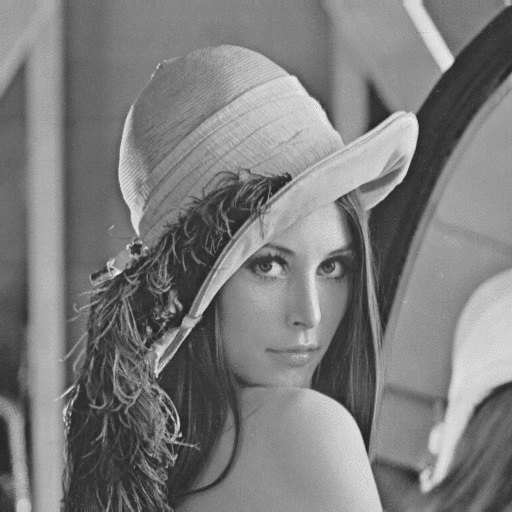}
	\caption{Example of image in Grayscale.}
\end{figure}

There are several possible ways to reduce a given image, as shown in the following example:

\begin{exem}\hypertarget{exem12}
	The image $$M= \left[\begin{tabular}{cccccc}
	0.8 & 0.7 & 0.2 & 1 & 0.5 & 0.5\\
	0.6 & 0.2 & 0.3 & 0.1 & 1 & 0\\
	0 & 0 & 0.6 & 0.4 & 0.9 & 1\\
	0.1 & 0.2 & 0.3 & 0.4 & 0.5 & 0.6\\
	\end{tabular}\right],$$
	can be reduced to another $2\times 3$ by partitioning $M$ in blocks $2\times 2$:
	$$\overline{M}= 
	\left[\begin{tabular}{ccc}
		$\left[\begin{tabular}{cc} 
			0.8 & 0.7\\ 0.6 & 0.2
		\end{tabular}\right]$ & 
		$\left[\begin{tabular}{cc} 
			0.2 & 1\\ 0.3 & 0.1
		\end{tabular}\right]$ & 
		$\left[\begin{tabular}{cc} 
			0.5 & 0.5\\  1 & 0 
		\end{tabular}\right]$\\
		& & \\
		$\left[\begin{tabular}{cc} 
			0 & 0\\ 0.1 & 0.2 
		\end{tabular}\right]$ & 
		$\left[\begin{tabular}{cc} 
			0.6 & 0.4\\ 0.3 & 0.4 
		\end{tabular}\right]$ &
		$\left[\begin{tabular}{cc} 
			0.9 & 1\\ 0.5 & 0.6 
		\end{tabular}\right]$
	\end{tabular}\right],$$
	 and applying to each block, for example, the function $f(x,y,z,w)=Max(x,y,z,w)$:
	 We obtain, the image:
	$$M_*=\left[\begin{tabular}{ccc}
	0.8 & 1 &\ 1\ \\
	0.2 & 0.6 &\ 1\ \\
	\end{tabular}\right]$$
	Applying $g(x,y,z,w)=Min(x,y,z,w)$ we  would obtain:
	$$M_{**}=\left[\begin{tabular}{ccc}
	0.2 & 0.1 & 0\\
	0 & 0.2 & 0.5\\
	\end{tabular}\right]$$
\end{exem}

In fact, if we apply any other function, we get a new image (usually different from the previous one), but what is the best?

One possible answer to this question involves a method called {\bf magnification} or {\bf extension} (see \cite{Jurio,Yang1,Yang2}), which is a method which magnifies the reduced image to another with the same size of the original one. The magnified image is then compared with the original input image.

\begin{exem}\hypertarget{exe13}
	From $M_*$ and $M_{**}$, we get images $4\times 6 $, $M'$ and $M''$, simply cloning each pixel ,
	$$\left[\begin{tabular}{c}
	x
	\end{tabular}\right]
	\longmapsto
	\left[\begin{tabular}{cc}
	x & x\\
	x & x\\
	\end{tabular}\right]$$
	We obtain, new images
	$$M_1=\left[\begin{tabular}{cccccc}
	0.8 & 0.8 & 1 & 1 &\ 1\ &\ 1\ \\
	0.8 & 0.8 & 1 & 1 &\ 1\ &\ 1\ \\
	0.2 & 0.2 & 0.6 & 0.6 &\ 1\ &\ 1\ \\
	0.2 & 0.2 & 0.6 & 0.6 &\ 1\ &\ 1\ \\
	\end{tabular}\right]$$
	and
	$$M_2= \left[\begin{tabular}{cccccc}
	0.2 & 0.2 & 0.1 & 0.1 &\ 0\ &\ 0\ \\
	0.2 & 0.2 & 0.1 & 0.1 &\ 0\ &\ 0\ \\
	\ 0\ &\ 0\ & 0.2 & 0.2 & 0.5 & 0.5\\
	\ 0\ &\ 0\ & 0.2 & 0.2 & 0.5 & 0.5\\
	\end{tabular}\right]
	$$
\end{exem}

Since $M_1$ e $M_2$ have the same size as the original image $M$, we can now measure what is the best reduction. This can be done by comparing the initial image $M$ with each of the resulting images, $M_1$ and $M_2$. But, how do  we compare?

One of the possibilities to compare the images $M_1$ and $M_2$ with the original image $M$ is to use the mensure PSNR \cite{Gonzales}, calculated as follows:
$$PSNR(I,K)=10\cdot log_{10}\left(\frac{MAX_I^2}{MSE(I,K)}\right),$$
where $I=I(i,j)$ and $K=K(i,j)$ are two images, $MSE(I,K)=\frac{1}{nm}\sum\limits_{i=1}^m\sum\limits_{j=1}^n[I(i,j)-K(i,j)]^n$ and $MAX_I$ is the maximum possible pixel value of pixel. Observe that the closer the image the smaller the value of MSE and the larger the value of PSNR \footnote{In particular, if the input image are equal, then the MSE value is zero and the PSNR will be infinity.}.

In what follows, we use $DYOW\!A$ operators: $\mathbf{H}, cOW\!A, Median$ and $Arith$ to reduce size of images in grayscale. We apply the following method:

\vspace{1em}
\hrule
\begin{center}
	{\bf Method 1}
\end{center}
\hrule
\begin{enumerate}
	\item Reduce the input images using the $\mathbf{H}$, $cOW\!A$, {\it Arithmetic Mean} and {\it Median};
	\item Magnify the reduced image to the size of the original image using the method described in \hyperlink{exe13}{example 13};
	\item Compare the last image with the original one using the measure $PSNR$.
\end{enumerate}
\hrule
\vspace{1em}
	
\begin{rmk}
	This general method can be applied to any kind of image. In this work  we applied it to the 10 images in grayscale of size $512\times 512$ (Figure \ref*{FigOrigImg}) \footnote{In this paper we made two reductions: using $2\times 2$ blocks and $4\times 4$ blocks.}.
\end{rmk}

\begin{figure}[!h]
	\centering
	\includegraphics[width=3cm]{fig1.jpg}
	\includegraphics[width=3cm]{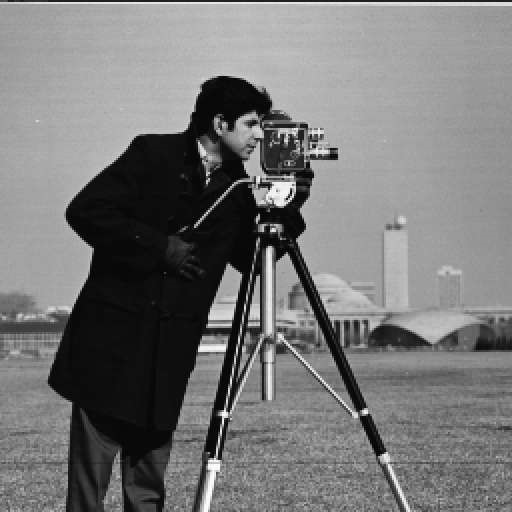}\\
    \vspace{0.05cm}
	\includegraphics[width=3cm]{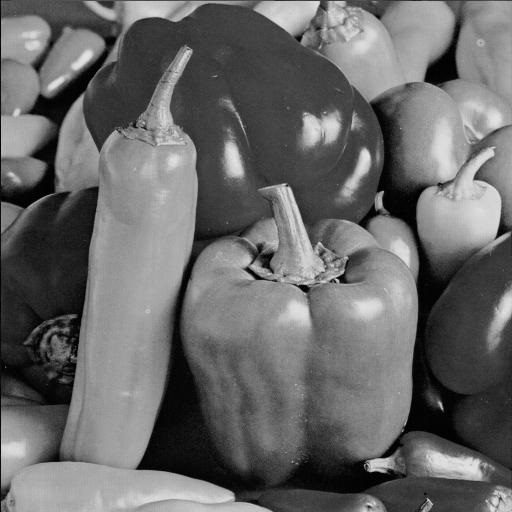}
	\includegraphics[width=3cm]{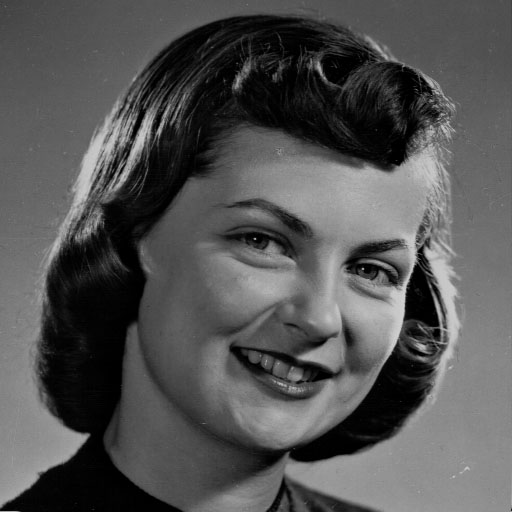}\\	
	\vspace{0.05cm}
	\includegraphics[width=3cm]{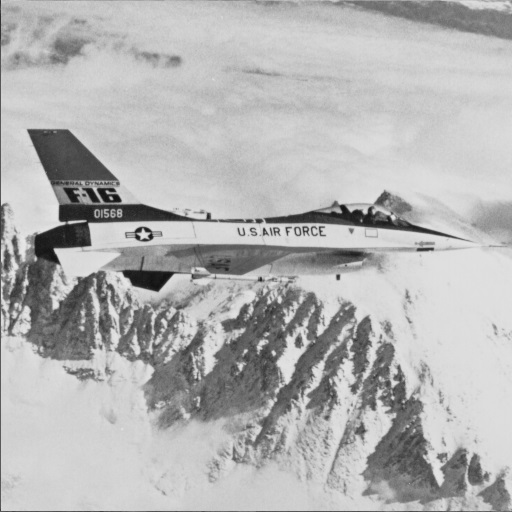}
	\includegraphics[width=3cm]{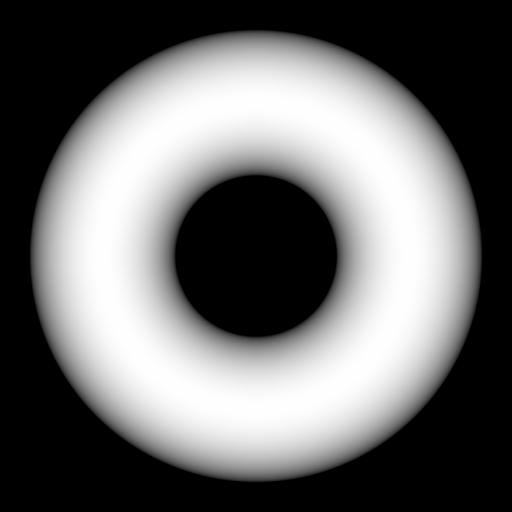}\\
	\vspace{0.05cm}
	\includegraphics[width=3cm]{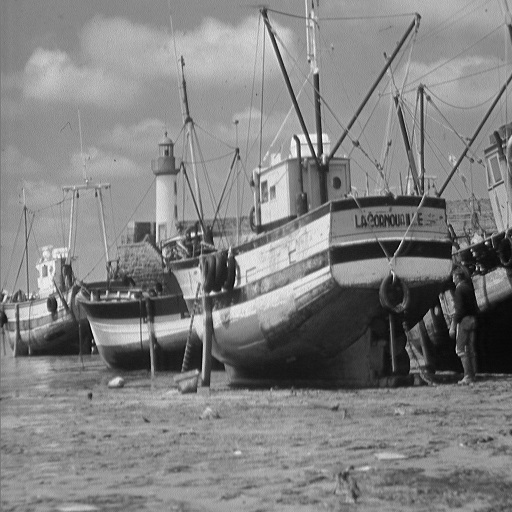}
	\includegraphics[width=3cm]{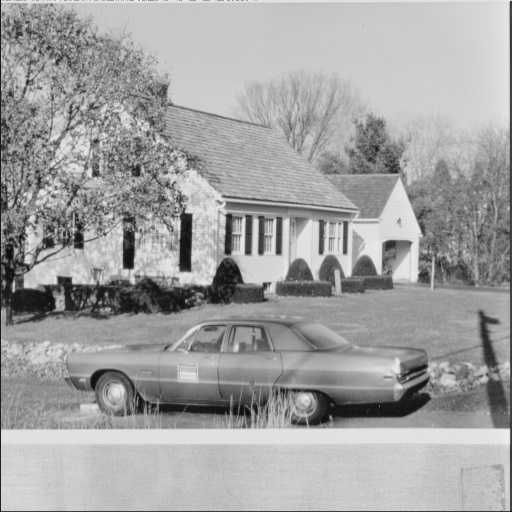}\\
	\vspace{0.05cm}
	\includegraphics[width=3cm]{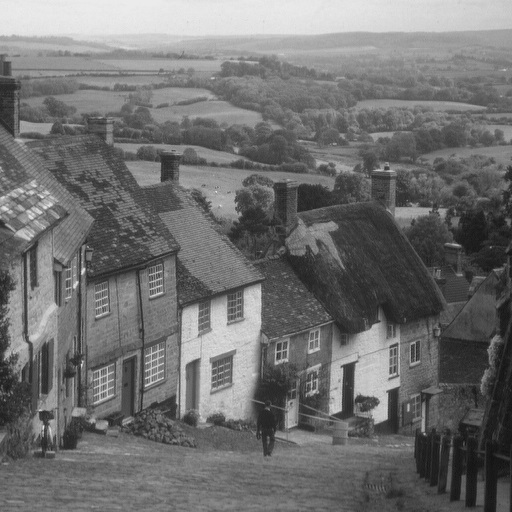}
	\includegraphics[width=3cm]{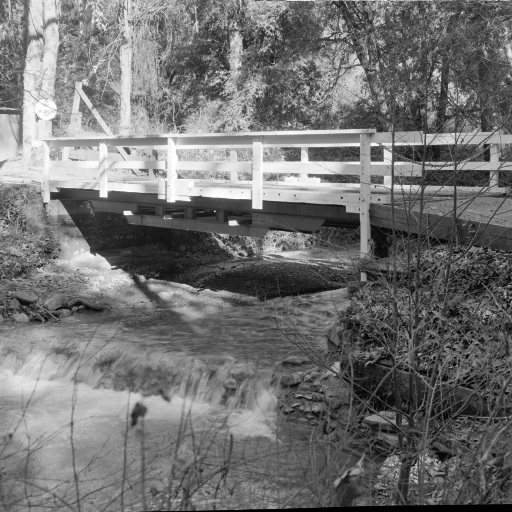}
	\caption{Original images}\label{FigOrigImg}
\end{figure}
 
In the tables I and II we present the PSNR values between the input images and the output provided by Method 1. Table 1 provides results for operators using blocks $2\times 2$ and Table II for blocks $4\times 4$.

\begin{table}[!htb]
\centering
\begin{tabular}{ccccc}
	\hline\
	& $\mathbf{H}$ & $cOW\!A$ &{\it Arith} & {\it Median}\\
	\hline\
	$\!$Img 01 & $29.63$ & $29.66$ & ${\bf 29,71}$ & $29.50$\\
	Img 02 & $33.15$ & $33.14$ & ${\bf 33.18}$ & $33.09$\\
	Img 03 & $29.52$ & $29.53$ & ${\bf 29.57}$ & $29.44$\\
	Img 04 & $31.54$ & $31.54$ & ${\bf 31.61}$ & $31.46$\\
	Img 05 & $27.87$ & $27.88$ & ${\bf 27.91}$ & $27.80$\\
	Img 06 & $40.78$ & $40.78$ & ${\bf 40.79}$ & $40.78$\\
	Img 07 & $27.40$ & $27.42$ & ${\bf 27.47}$ & $27.30$\\
	Img 08 & $26.56$ & $26.57$ & ${\bf 26.61}$ & $26.47$\\
	Img 09 & $28.84$ & $28.85$ & ${\bf 28.89}$ & $28.73$\\
	Img 10 & $24.43$ & $24.45$ & ${\bf 24.53}$ & $24.27$\\
	\hline\
	Avg & $29.97$ & $29.98$ & ${\bf 30.03}$ & $29.88$	
\end{tabular}
\caption{$PSNR$ values after a reduction using the $DYOW\!A$s operators using blocks $2\times 2$}
\end{table}

\begin{table}[!htb]
	\centering
	\begin{tabular}{ccccc}
		\hline\
		& $\mathbf{H}$ & $cOW\!A$ &{\it Arith} & {\it Median}\\
		\hline\
		$\!$Img 01 & $26.34$ & $26.26$ & $26.36$ & ${\bf 26.70}$\\
		Img 02 & $23.64$ & $23.60$ & ${\bf 23.65}$ & $22.78$\\
		Img 03 & $25.55$ & $25.46$ & ${\bf 25.56}$ & $24.84$\\
		Img 04 & $27.53$ & $27.45$ & ${\bf 27.54}$ & $26.86$\\
		Img 05 & ${\bf 24.14}$ & $24.06$ & ${\bf 24.14}$ & $23.28$\\
		Img 06 & $34.39$ & $34.34$ & ${\bf 34.41}$ & $33.83$\\
		Img 07 & $23.98$ & $23.88$ & ${\bf 23.99}$ & $23.19$\\
		Img 08 & ${\bf 23.07}$ & $22.97$ & ${\bf 23.07}$ & $22.18$\\
		Img 09 & $25.78$ & $25.69$ & ${\bf 25.79}$ & $25.05$\\
		Img 10 & ${\bf 21.71}$ & $21.61$ & ${\bf 21.71}$ & $20.62$\\
		\hline\
		Avg & $25.61$ & $25.53$ & ${\bf 25.62}$ & $24.83$
	\end{tabular}
	\caption{$PSNR$ values after a reduction using the $DYOW\!A$s operators using blocks $4\times 4$.}
\end{table}

According to PSNR, $Arith$ provided the higher quality image. However, the reduction operators generated by $\mathbf{H}$ and $cOW\!A $ provide us quite similar images to those given by $Arith$.

Observe that although the Method 1 is very simple, it introduces noise in the resulting image. In what follows we show that the operator $\mathbf{H}$ is suitable to filter images with noise. This is done by using $\mathbf{H}$ to define the weights which are used in the process of convolution. This new process will, then, be used to provide a better comparison in the Method 1.

\section{$DYOW\!A$'s as Tools of Noise Reduction}

In this section we show that the $DYOW\!A$ operators studied in section III can be used to deal with images containing noise.

The methodology employed here consists to analyze the previous images with Gaussian noise $\sigma = 10\%$ and $15\%$; apply a filter built upon the operators $\mathbf{H}$, $cOW\!A$ and $Arith$ based on convolution method (See \cite{Gonzales}), and compare the resulting images with the original one using PSNR.

\begin{table}[!htb]
	\centering
	\begin{tabular}{ccccc}
		\hline\
		& $\mathbf{H}$ & $cOW\!A$ &{\it Arith}  & No Tratament\\
		\hline\
		$\!$Img 01 & ${\bf 30.96}$ & $30.56$ & ${\bf 30.96}$ & $23.83$\\
		Img 02 & ${\bf 28.16}$ & $27.78$ & $27.36$ & $24.36$\\
		Img 03 & ${\bf 31.33}$ & $30.99$ & $31.08$ & $24.23$\\
		Img 04 & ${\bf 32.33}$ & $32.09$ & $32.20$ & $24.48$\\
		Img 05 & ${\bf 30.39}$ & $30.09$ & $30.10$ & $24.06$\\
		Img 06 & $31.66$ & ${\bf 31.69}$ & $31.38$ & $25.73$\\
		Img 07 & ${\bf 28.97}$ & $28.65$ & $28.80$ & $23.93$\\
		Img 08 & ${\bf 28.51}$ & $28.28$ & $28.25$ & $24.02$\\
		Img 09 & ${\bf 30.03}$ & $29.73$ & $30.02$ & $23.91$\\
		Img 10 & ${\bf 25.97}$ & $25.84$ & $25.81$ & $23.76$\\
		\hline\
		Avg & ${\bf 29.83}$ & $29.57$ & $29.60$ & $24.23$\\
	\end{tabular}
	\caption{$PSNR$ values between the output image with original one, in which $\sigma=10\%$}
\end{table}

\begin{table}[!htb]
	\centering
	\begin{tabular}{ccccc}
		\hline\
		& $\mathbf{H}$ & $cOW\!A$ &{\it Arith}  & No Tratament\\
		\hline\
		$\!$Img 01 & $29.88$ & $29.40$ & ${\bf 29.97}$ & $21.19$\\
		Img 02 & ${\bf 27.33}$ & $27.04$ & $26.74$ & $21.48$\\
		Img 03 & ${\bf 29.92}$ & $29.55$ & $29.79$ & $21.32$\\
		Img 04 & ${\bf 30.23}$ & $30.06$ & $30.08$ & $21.51$\\
		Img 05 & ${\bf 29.38}$ & $28.96$ & $29.27$ & $21.30$\\
		Img 06 & $27.95$ & ${\bf 28.03}$ & $27.56$ & $22.36$\\
		Img 07 & ${\bf 28.23}$ & $27.84$ & $28.14$ & $21.26$\\
		Img 08 & ${\bf 27.87}$ & $27.57$ & $27.70$ & $21.28$\\
		Img 09 & $29.17$ & $28.76$ & ${\bf 29.22} $ & $21.25$\\
		Img 10 & ${\bf 25.55}$ & $25.39$ & $25.44$ & $21.41$\\
		\hline\
		Avg & ${\bf 28.55}$ & $28.26$ & $28.39$ & $21.44$\\
	\end{tabular}
	\caption{$PSNR$ values between the output image with original one, in which $\sigma=15\%$.}
\end{table}

Tables III and IV demonstrate the power of $\mathbf{H}$ on images with noise. All listed operators improved significantly the quality of the image with noise. However, $\mathbf{H}$ exceeded all other analyzed.

Figure 4 shows an example of a image with Gaussian noise $\sigma=15\%$ and the Figure 5 the output image after applying the filter of convolution using $\mathbf{H}$.

\begin{figure}[!htb]
	\centering
	\includegraphics[width=5.0cm]{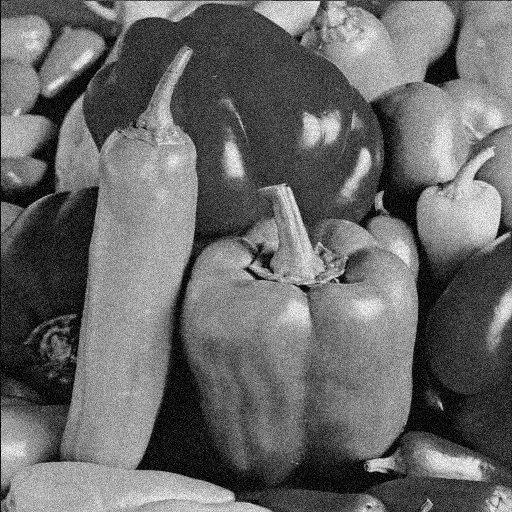}
	\caption{Image 03 with Gaussian noise $\sigma=15\%$.}
\end{figure}

\begin{figure}[!htb]
	\centering
	\includegraphics[width=5.0cm]{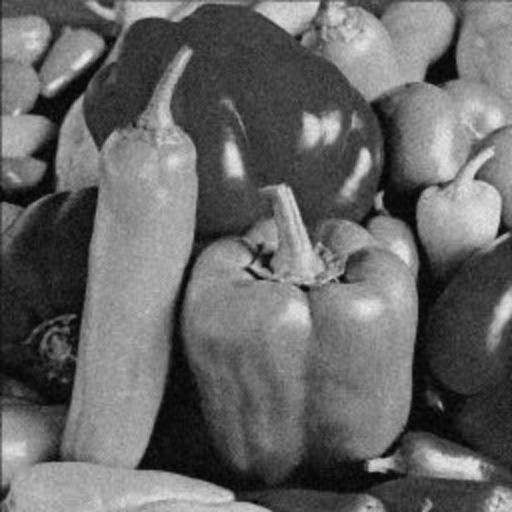}
	\caption{Figure 4 after being treated with $H$ by convolution.}
\end{figure}

The reader can see in tables III and IV that $\mathbf{H}$ proved to be an excellent operator for noise reduction. 

In what follows, we modify the Method 1 in order to provide a better magnified image to be compare with the original one.

\vspace{1em}
\hrule
\begin{center}
	{\bf Method 1'}
\end{center}
\hrule
\begin{enumerate}
	\item Reduce the input images using the $\mathbf{H}$, $cOW\!A$, {\it Arithmetic Mean} and {\it Median};
	\item (A) Magnify the reduced image to the size of the original image using the method described in \hyperlink{exe13}{Example 13}, and (B) Use the convolution filter, using $\mathbf{H}$, on the last image;
	\item Compare the last image with the original one using the measure $PSNR$.
\end{enumerate}
\hrule
\vspace{1em}

Tables V and VI show the obtained results:

\begin{table}[!htb]
	\centering
	\begin{tabular}{ccccc}
		\hline\
		& $\mathbf{H}$ & $cOW\!A$ &{\it Arith} & {\it Median}\\
		\hline\
		$\!$Img 01 & $30.60$ & $30.55$ & ${\bf 30.62}$ & $30.50$\\
		Img 02 & $28.68$ & $28.67$ & ${\bf 28.70}$ & $28.68$\\
		Img 03 & ${\bf 30.89}$ & $30.83$ & $30.85$ & $30.84$\\
		Img 04 & ${\bf 32.74}$ & $32.71$ & $32.71$ & $32.71$\\
		Img 05 & ${\bf 29.22}$ & $29.18$ & $29.16$ & $29.19$\\
		Img 06 & $43.88$ & $43.80$ & ${\bf 43.93}$ & $43.90$\\
		Img 07 & ${\bf 28.05}$ & $28.01$ & $28.03$ & $27.95$\\
		Img 08 & ${\bf 27.39}$ & $27.37$ & $27.37$ & $27.36$\\
		Img 09 & $29.36$ & $29.35$ & ${\bf 29.40}$ & $29.30$\\
		Img 10 & $24.94$ & $24.92$ & ${\bf 24.95}$ & $24.83$\\
		\hline\
		Avg & ${\bf 30.57}$ & $30.54$ & ${\bf 30.57}$ & $30.53$
	\end{tabular}
	\caption{$PSNR$ between the original image and the magnified image from the image reduced by blocks $2\times 2$.}
\end{table}

\begin{table}[!htb]
	\centering
	\begin{tabular}{ccccc}
		\hline\
		& $\mathbf{H}$ & $cOW\!A$ &{\it Arith} & {\it Median}\\
		\hline\
		$\!$Img 01 & ${\bf 27.44}$ & $27.41$ & $27.41$ & $27.01$\\
		Img 02 & ${\bf 23.91}$ & $23.88$ & ${\bf 23.91}$ & $23.24$\\
		Img 03 & ${\bf 26.86}$ & $26.85$ & $26.85$ & $26.28$\\
		Img 04 & ${\bf 28.87}$ & $28.86$ & $28.85$ & $28.39$\\
		Img 05 & ${\bf 25.15}$ & ${\bf 25.15}$ & $25.12$ & $24.64$\\
		Img 06 & ${\bf 28.13}$ & $28.05$ & ${\bf 28.13}$ & $27.04$\\
		Img 07 & ${\bf 24.68}$ & $24.63$ & ${\bf 24.68}$ & $24.13$\\
		Img 08 & ${\bf 23.78}$ & $23.76$ & ${\bf 23.78}$ & $23.14$\\
		Img 09 & ${\bf 26.45}$ & $26.40$ & ${\bf 26.45}$ & $25.92$\\
		Img 10 & ${\bf 22.27}$ & $22.21$ & $22.26$ & $21.49$\\
		\hline\
		Avg & ${\bf 25.75}$ & $25.72 $ & $ 25.74$ & $25.12$
	\end{tabular}
	\caption{$PSNR$ between the original image and the magnified image from the image reduced by blocks $4\times 4$}
\end{table}

Since the output of convolution using $\mathbf{H}$ is closer to the original input image, the tables V and VI show that the process of reduction using $\mathbf{H}$ is more efficient.

\section{Final Remarks}

In this paper we propose a generalized form of Ordered Weighted Averaging function, called \textbf{Dynamic Ordered Weighted Averaging} function or simply \textbf{DYOWA}. This functions are defined by weights, which are obtained dynamically from of each input vector ${\bf x}\in [0,1]^n$. We demonstrate, among other results, that $OW\!A$ functions are instances of $DYOW\!A$s, and, hence, functions like: {\it Arithmetic Mean}, {\it Median}, {\it Maximum}, {\it Minimum} and $cOW\!A$ are also examples of $DYOW\!A$.

In the second part of this work we present a particular $DYOW\!A$, called of $\mathbf{H}$, and show that it is idempotent, symmetric, homogeneous, shift-invariant, and moreover, it has no zero divisors and one divisors, and also does not have neutral elements. Since aggregation functions which satisfy these properties are extensively used in image processing, we tested its usefulness to: (1) reduce the size of images and (2) deal with noise in images.

In terms of image reduction, Method 1 showed a weakness, since it adds noise during the process of magnification. However, the treatment of noise with function \textbf{H} improved the magnification step providing an evidence that the function \textbf{H} is more efficient to perform the image reduction process.

\begin{figure}[!htb]
	\centering
	\includegraphics[width=6.5cm]{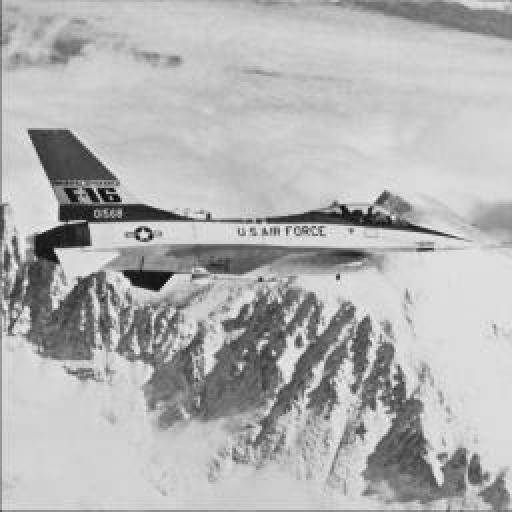}
	\caption{Magnification of image 06 reduce by bloks $2\times 2$ using the operator $H$ by Method 1}
\end{figure}

\begin{figure}[!htb]
	\centering
	\includegraphics[width=6.5cm]{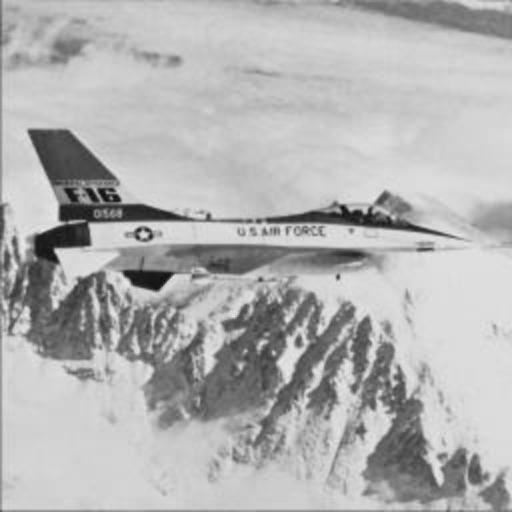}
	\caption{Magnification of image 06 reduce by bloks $2\times 2$ using the operator $H$ by Method 1'}
\end{figure}

\begin{figure}[!htb]
	\centering
	\includegraphics[width=6.5cm]{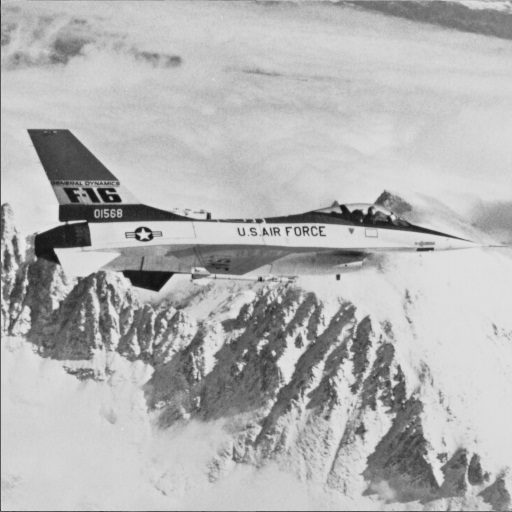}
	\caption{Image 06}
\end{figure}

\newpage

\begin{figure}[!htb]
	\centering
	\includegraphics[width=6.5cm]{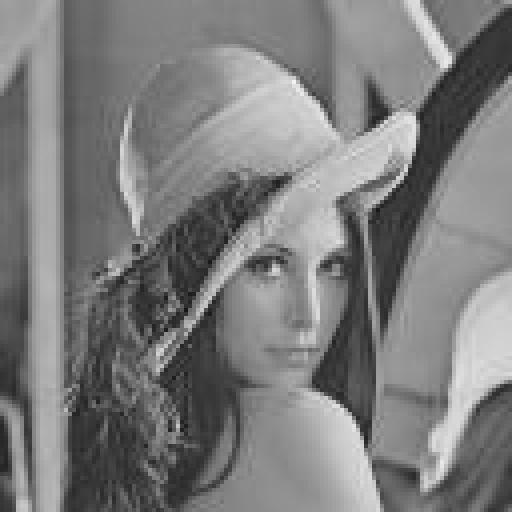}
	\caption{Magnification of image 01 reduce by bloks $4\times 4$ using the operator $H$ by Method 1}
\end{figure}

\begin{figure}[!htb]
	\centering
	\includegraphics[width=6.5cm]{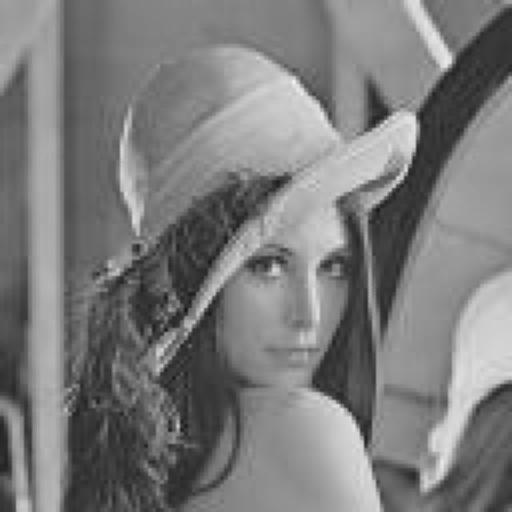}
	\caption{Magnification of image 01 reduce by bloks $4\times 4$ using the operator $H$ by Method 1'}
\end{figure}

\begin{figure}[!htb]
	\centering
	\includegraphics[width=6.5cm]{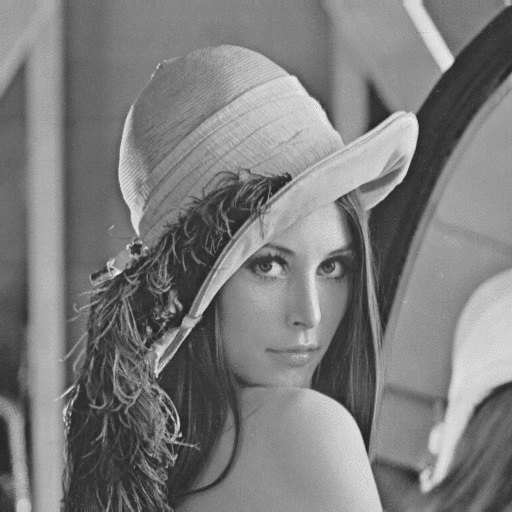}
	\caption{Image 01}
\end{figure}

\hfill\\


\newpage

\begin{thebibliography}{9}
	
	\bibitem{Patermain} D. Paternain, J. Fernandeza, H. Bustince, R. Mesiar ,G. Beliakov, Construction of image reduction operators using averaging aggregation functions, Fuzzy Sets and Systems 261 (2015) 87–111.
	
	\bibitem{Joseph} R. P. Joseph, C. S. Singh, M. Manikandan, Brain Tumor MRI Image Segmentation and Detection in Image Processing, International Journal of Research and Tecnology, vol 3 (2014), ISSN: 2319-1163.
	
	\bibitem{Solanki} A. J. Solanki, K. R. Jain, N. P. Desai, ISEF Based Identification of RCT/Filling in Dental Caries of Decayed Tooth, International Journal of Image Processing (IJIP), vol 7 (2013) 149-162.
	
	\bibitem{Gonzales} R. C. Gonzales, R. E. Woods, Digital Image Processing, third edition, Pearson, 2008.
	
	\bibitem{Yager1} R. R. Yager, Ordered weighted averaging aggregation operators in multicriteria decision making, IEEE Trans. Syst. ManCybern. 18 (1988) 183-190.
	
	\bibitem{Dubois1} D. Dubois, H. Prade, On the use of aggregation operations in information fusion processes, Fuzzy Sets Syst. 142 (2004) 143-161.
	
	\bibitem{Chen} S.-J. Chen and C.-L. Hwang. Fuzzy Multiple Attribute Decision Making: Methods and Applications. Springer, Berlin, Heidelberg, 1992.
	
	\bibitem{Zhou} S. -M. Zhou, F. Chiclana, R. I. John, J. M. Garibaldi, Type - 1 OWA operators for aggregating uncertain information with uncertain weight sinduced	by type -2 linguistic quantifiers, Fuzzy Sets Syst. 159 (2008) 3281-3296.
	
	\bibitem{Yager2} R. R. Yager, G. Gumrah, M. Reformat, Using a web personal evaluation tool — PET for lexicographic multi-criteria service selection, Knowl. - Based Syst. 24 (2011) 929-942.
	
	\bibitem{Patermain2} D. Paternain, A. Jurio, E. Barrenechea, H. Bustince, B. C. Bedregal, E. Szmidt: An alternative to fuzzy methods in decision-making problems. Expert Syst. Appl. 39(9): 7729-7735 (2012).
	
	\bibitem{Bustince} H. Bustince, M. Galar, B. C. Bedregal, A. Koles\'arov\'a, R. Mesiar: A New Approach to Interval-Valued Choquet Integrals and the Problem of Ordering in Interval-Valued Fuzzy Set Applications. IEEE T. Fuzzy Systems 21(6): 1150-1162 (2013).
	
	\bibitem{Beliakov2} G. Beliakov, H. Bustince, D. Paternain, Image reduction using means on discrete product lattices, IEEETrans. ImageProcess. 21 (2012) 1070-1083.
	
	\bibitem{Liang} X. Liang, W. Xu, Aggregation method for motor drive systems, Eletric Power System Research 117 (2014) 27-35.
	
	\bibitem{Dubois2} D. Dubois, H. Prade (Eds.), Fundamental sof Fuzzy Sets, Kluwer Academic Publishers, Dordrecht, 2000.
	
	\bibitem{Beliakov} G. Beliakov, A. Pradera, T. Calvo, Aggregation functions: a guide for practitioners, Stud. Fuzziness Soft Comput. 221, 2007.

	\bibitem{Grabisch} M. Grabisch, E. Pap, J. L. Marichal, R. Mesiar, Aggregation Functions, University Press Cambridge, 2009.
	
	\bibitem{Baczynski} M. Baczy\'nnski, B. Jayaram, Fuzzy Implications, Springer, Berlin, 2008.
		
	\bibitem{Yager3} R. R. Yager, Centered OWA operators, Soft Comput. 11 (2007) 631-639.
	
	\bibitem{Jurio} A. Jurio, M. Pagola, R. Mesiar, G. Beliakov, H. Bustince, Image magnification using interval information, IEEE Trans. Image Process. 20 (2011) 3112-3123.
	
	\bibitem{Yang1} J. Yang, J. Wright, T. S. Huang, Y. Ma, Image Super-Revolution Via Sparse Representation, IEEE Trans on Image Processing, v. 19, n. 11 (2010) 2861-2873.
	
	\bibitem{Yang2} J. Yang, J. Wright, T. S. Huang, Y. Ma, Image Super-Revolution as Sparse Representation of Raw Image Patches, Computer Vision and Pattern Recoginition, 2008. CVPR 2008. IEEE Conference,  IEEE, 2008, 1-8.
	
\end{thebibliography}
\end{document}